\documentclass{article} 
\PassOptionsToPackage{numbers,sort&compress}{natbib}
\usepackage{natbib}
\usepackage[left=1in,top=1in,right=1in,bottom=1in,letterpaper]{geometry}

\usepackage{amsmath}
\usepackage{amssymb}
\usepackage{amsthm}
\usepackage{microtype}
\usepackage{hyperref}
\usepackage{url}
\usepackage{wrapfig}
\usepackage{booktabs}
\usepackage{microtype}
\usepackage{graphicx}
\usepackage{subcaption}
\usepackage{tikz}
\usetikzlibrary{matrix}
\usepackage{adjustbox}
\usepackage{booktabs}
\usepackage[utf8]{inputenc} 
\usepackage{amsfonts}       
\usepackage{nicefrac}       
\usepackage{enumitem}
\usepackage{float}

\usepackage{lineno}

\newtheorem{theorem}{Theorem}[section]

\definecolor{darkblue}{rgb}{0, 0, 0.5}
\hypersetup{colorlinks=true, citecolor=darkblue, linkcolor=darkblue, urlcolor=darkblue}

\title{Assign and Add:\\ A Mechanistic Study of Compositional Arithmetic}

\author{
Brady Exoo\\
Yale University
\and
Alberto Bietti\\
Flatiron Institute
\and
John Sous\\
Yale University
}

\begin{document}

\maketitle

\begin{abstract}
Large language models are able to compose skills in order to perform complex tasks, many of which might not have been seen during training. The details of how exactly this composition occurs remain elusive. In this paper, we study a mechanism for compositional generalization in transformers by considering a simple controlled setting involving variable assignment and modular addition. By partitioning our training data into disjoint sets, we observe that small transformers are able to generalize to previously unseen combinations of variables and numbers. Our mechanistic analysis shows that the same ``modular addition'' MLP module is used whether the inputs are given directly or indirectly through a separate variable assignment mechanism. We also analyze the training dynamics from an empirical lens, which reveals three phases of learning: first, modular addition is learned, then the structure required for variable assignment, and finally a refinement phase where the model generalizes to some hard sequences not seen in training. Finally, we provide a theoretical framework to explain how compositionality emerges from training dynamics. These results suggest that compositional generalization can be a natural consequence of the compositionality of internal mechanisms in~transformers.
\end{abstract}

\section{Introduction}
The reliability of large language models on reasoning tasks depends critically on their ability to generalize across variations of problems they might have seen during training. A central mechanism for their ability (or failure) to generalize is \emph{compositionality}~\cite{lake2018scan,kim2020cogs,press2023compositionalitygap,dziri2023faith}: tackling complex tasks as a composition of small atomic mechanisms provides a reliable way to generalize, similar to how a small valid program is guaranteed to work on arbitrary inputs.
While the generalization capabilities of modern LLMs are improving, a deeper understanding of how they learn to generalize via composition is still missing, and could further improve their reliability.
.
Studying mechanisms in transformers and their training dynamics in controlled settings via mechanistic interpretability provides a framework for understanding how these generalizing capabilities arise in transformers.
Our goal in this paper is to systematically study how training a transformer on a simple compositional task may lead to compositionality of the mechanisms in its architecture, and how this enables out-of-distribution generalization. Specifically, we take two well-studied mechanisms with highly different behaviors---variable binding and modular addition---and consider a task involving their composition. Variable binding often relies on an induction head circuit~\citep{elhage2021mathematical,olsson2022context}, which implements an associative recall mechanism, a fundamental operation even in large models which has been well-studied from a training dynamics standpoint using associative memories~\citep{bietti2023birth,reddymechanistic,nichani2024transformers}. Modular arithmetic is a mathematical task which typically leads to highly structured representations involving Fourier features, and often exhibits ``grokking'' in its training dynamics --- a phenomenon where a model initially overfits and memorizes the training data before eventually transitioning to generalizable reasoning~\citep{power2022grokking,nanda2023progress,gromov2023grokking,he2026mechanism}. Since mathematical reasoning in scientific domains often combines information retrieval with arithmetic operations, this raises the following questions:
\vspace{-2mm}
\begin{center}
\it
    How do transformers learn to compose variable binding and modular addition mechanisms to generalize beyond the training distribution? How do these capabilities emerge during training?
\end{center}

In this paper, we study these questions in a controlled setting, revealing how transformers learn to compose variable binding with arithmetic ability. Specifically, we analyze a 2-layer transformer trained to perform modular addition where the operands may be variables that were dynamically assigned in the context.

Our main contributions are as follows: 
\begin{itemize}[leftmargin=*,topsep=-2pt,noitemsep]
    \item We present a tractable model for mechanistically interpreting compositional arithmetic induced by composing variable assignment with modular arithmetic.
    \item We analyze the training dynamics, revealing that modular arithmetic is learned prior to and independently of variable assignment. We find that the model attains near-perfect test accuracy on $0$-variable modular addition while still learning to perform variable assignment.
    \item We provide a thorough mechanistic study of the model at the end of training, and show that generalization arises via composition of the modular addition circuit, implemented by the last MLP, and a modified induction head circuit, whereby the second layer attends to the numbers assigned to the correct variables involved in the addition. 
    \item We provide a theoretical analysis of how gradient dynamics allow the composition of the variable binding and modular addition circuits.
\end{itemize}

\paragraph{Related work.}
The problem of composition in mechanistic interpretability has been studied on pretrained models via circuit analysis~\citep[e.g.,][]{wang2022interpretability, hanna2023does,ameisen2025circuit}, as well as on small transformers in synthetic controlled settings involving, e.g., variable binding, state tracking, arithmetic, or in-context learning~\citep{zhang2022unveiling,liu2023transformers,nanda2023progress,davies2023variablebinding,ramesh2024compositionalcapabilities,he2024learningtogrok, wu2025transformers,zhao2026shatteredcompositionalitycounterintuitivelearning}. Compared to these works, our focus is on how composition arises during training in a controlled setting, and on a detailed understanding---at the weights level---of the composition of two very different types of mechanisms (associative recall via induction heads and modular addition via non-linear Fourier features), backed by a theoretical analysis of gradient dynamics. Induction heads were introduced and analyzed empirically in~\citep{elhage2021mathematical,olsson2022context}, after which theoretical work revealed that optimization dynamics implements associative memory through the induction head mechanism~\citep{vonoswald2023transformers, bietti2023birth}. This connection between associative memory and induction heads has received interest in the context of solving recursions~\citep{cabannes2024iteration}. Grokking was first discovered in the context of small algorithmic datasets~\citep{power2022grokking}. Later, mechanistic models based on one or few layer architectures have revealed the emergence of Fourier features in the grokking learning process~\citep{nanda2023progress, gromov2023grokking}, and recent works have studied their training dynamics theoretically~\cite{tian2025composingglobalsolutionsreasoning,kunin2026alternating,he2026mechanism}.
Ref.~\citep{wang2025easytohard} studies training dynamics on a compositional multi-hop reasoning task, but their theoretical setup is idealized and not directly related to mechanistic analysis of a generically trained model.

\section{Background}
This section reviews preliminaries on the main concepts we use in this work.

\paragraph{Transformer architecture.}
Our architecture utilizes a 2-layer, single-head transformer that deliberately omits the layer-1 MLP. A single-head transformer \citep{Vaswani+2017} maps each token in a vocabulary of $n$ discrete tokens to a $d_{\mathrm{model}}$-dimensional vector via the embedding matrix $W_E$. For a detailed exposition of the transformer architecture from a mechanistic-interpretability perspective, see~\citet{elhage2021mathematical}. Following the mechanistic interpretability framework established in \citep{elhage2021mathematical}, we conceptually separate the operations of each attention head into two distinct linear circuits. The Query-Key (QK) circuit, characterized by the weight matrices $W_Q$ and $W_K$, computes the attention pattern matrix and determines where information is routed across the sequence. Simply put, it calculates how much the current token should ``listen to'' or draw information from each of the preceding tokens. The Output-Value (OV) circuit, characterized by $W_V$ and $W_O$, determines what information is extracted from the source token and written into the destination token's residual stream. Following an attention layer, the residual stream is typically passed through an MLP, which operates on each token position independently to compute non-linear transformations. Therefore, in our architecture, the residual stream only undergoes this non-linear MLP transformation following the second attention layer. Once the sequence has been processed through all attention and MLP sublayers, the unembedding matrix, $W_U$, linearly projects the final residual stream vector of the last token back into the vocabulary space to yield the output logits. A softmax transformation is applied to normalize the logits, and the model's final prediction is the argmax of the normalized logits.

\paragraph{Induction head mechanisms and associative memories.}
Induction heads are a common mechanism found in transformers with at least $2$ attention layers that support in-context learning \citep{olsson2022context}. Generally, these heads implement a copying mechanism that takes inputs of the form \texttt{[$\cdots$ A B $\cdots$ A]} and predicts \texttt{B}. In our model, rather than copying a token, we find a modified induction-head mechanism that retrieves assigned variables and routes their assigned values into the downstream MLP to compute modular addition.
Following~\cite{bietti2023birth,dar2023analyzing}, our mechanistic analysis leverages the view of attention matrices as associative memories that capture associations between pairs of embeddings, leading to specific attention behaviors.

\paragraph{Modular addition with Fourier features.}
Our analysis leverages the findings of previous literature on modular addition, namely that neural networks tend to learn structured Fourier representations to solve this task. This is typically done with input and output weights of an MLP that are sinusoidal, in a way that summing neurons for a given label yields constructive interference~\cite{gromov2023grokking}. Such a generalizing solution often appears after an initial memorization phase, a typical example of grokking.

\section{Setup}
\label{sec:setup}

This section presents the data distribution and model architecture we utilize, with the goal of designing a task that requires  composition of  at least two different mechanisms (specifically, modular addition and variable binding), while remaining tractable to analyze mechanistically.
Our code is available at~\href{https://github.com/bexoo/assign-and-add/}{MechCompose}.

\paragraph{Model.}
We train a $2$-layer transformer with a single attention head,  $d_{\mathrm{model}} = d_\mathrm{head}=128$, and $d_\mathrm{mlp}=512$, using GPT-2--style absolute positional embeddings~\citep{radford2019language}, implemented in TransformerLens~\citep{nanda2022transformerlens}, an open-source library for mechanistic interpretability released under the MIT License. To simplify circuit discovery, we omit biases and normalization layers. We also remove the layer-$1$ MLP, as we find it unnecessary to perform the task. We implement attention using causal masking. Our model is trained with AdamW~\citep{loshchilov2018decoupled} using a learning rate of $10^{-3}$ and weight decay of~$2 \times 10^{-2}$. 
While our mechanistic analysis focuses on a single training run, we expect our findings to hold across different initializations.

\newcommand{\tokrows}[1]{%
\begin{adjustbox}{max width=\columnwidth,center}
\begin{tikzpicture}[baseline]
\tikzset{
  tok/.style={
    draw,
    line width=0.4pt,
    minimum height=2.6ex,
    text width=2.2em, 
    inner sep=1pt,
    font=\small,
    text height=1.6ex,
    text depth=0.25ex,
    align=center      
  },
  ans/.style={
    fill=blue!15 
  }
}
\matrix[
  matrix of nodes,
  ampersand replacement=\&,
  nodes={tok},
  nodes in empty cells,
  row sep=4pt,
  column sep=-\pgflinewidth
]{
#1
};
\end{tikzpicture}
\end{adjustbox}
}

\begin{figure}[t]
\centering
\tokrows{
PAD \& PAD \& PAD \& PAD \& c \& 1 \& PAD \& f \& 17 \& a \& 42 \& PAD \& + \& f \& 19 \& = \& |[ans]| 36 \\
b \& 2 \& PAD \& g \& 08 \& PAD \& c \& 53 \& PAD \& PAD \& e \& 14 \& + \& 17 \& 32 \& = \& |[ans]| 49 \\
a \& 3 \& c \& 58 \& PAD \& d \& 19 \& f \& 06 \& b \& 47 \& PAD \& + \& b \& f \& = \& |[ans]| 53 \\
}
\caption{Examples of tokenized sequences formed of variables and constants. The final token in blue represents the target answer, and is not present in the training sequence. Note that the third sequence will be in the test set, since it has $b$ in the first position of the addition.}
\vspace{-4mm}
\label{fig:token-seqs}
\end{figure}

\paragraph{Data distribution.} 
We consider a problem of variable assignment and modular addition, with sequences such as \texttt{c=1,b=17,a=42,b+19=?}, where addition is performed modulo $N=59$ over $V=12$ distinct variables. Each sequence of size $T=16$ is tokenized as shown in Figure~\ref{fig:token-seqs}, with random padding added to prevent solutions that rely explicitly on absolute position. Note also that our tokenization omits \texttt{=} in the variable assignments and the \texttt{?} at the end of the sequence, and represents addition in a modified order: due to causal masking, we place the \texttt{+} token before the operands to signal the start of the addition operation.

   \begin{figure}[t]
    \centering
    \includegraphics[ width=0.65\columnwidth]{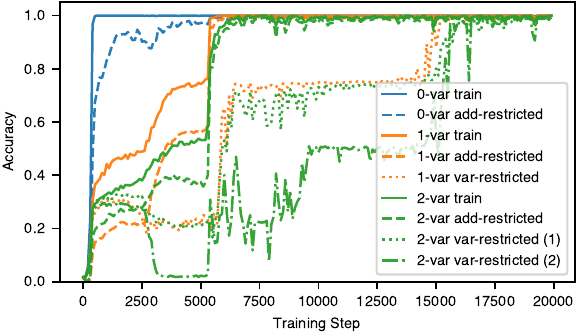}
    \caption{\textbf{Accuracies during training partitioned by evaluation set.} Add-restricted sets contain sequences with held-out addition pairs as discussed in Section \ref{sec:setup}, and var-restricted sets contain sequences with held-out variable positions. The (1) and (2) for the 2-var var-restricted sets denote how many of the variables are in ``bad'' positions. Train sets are those with both valid addition pairs and variable positions. An alternate run is displayed in Appendix~\ref{appx:supplementary} Figure~\ref{fig:accuracies2}.}
    \label{fig:accuracies}
\end{figure}

To evaluate the model's ability to generalize, we remove approximately $30\%$ of the $59^2$ possible addition pairs from the training data. Concretely, if $(1,3)$ is removed from the training set, then all sequences containing \texttt{1+3}; \texttt{c=1,c+3}; and \texttt{c=3,1+c} are also removed. We also constrain certain variables to fixed positions in the addition. In particular, in the training set the variables $a$ and $b$ are restricted from appearing in the first position of the addition (position $13$ of the sequence) and $g$ and $h$ are restricted from appearing in the last position of the addition (position $14$ of the sequence).

Let $S$ denote the set of all valid sequences. We partition $S$ into $S_0$, $S_1$, and $S_2$, the sets of sequences in which the addition contains $0$, $1$, and $2$ variables, respectively. We also partition $S$ according to the restrictions imposed on the training set. Let $S_V$ be the set of sequences in which variables appear only in valid positions, and let $S_A$ be the set of sequences with valid addition constant pairs. Our training set is then $S_T \;=\; S_V \cap S_A$. Finally, we generate sequences satisfying these constraints on-the-fly during training, rather than pre-partitioning a fixed set of sequences into training and test~splits.

In what follows, we analyze our results with respect to $0$-variable ($0$-var), $1$-variable ($1$-var), and $2$-variable ($2$-var) tasks.

\section{Generalization Results}
We first describe several ways in which our model generalizes.

\paragraph{Training results.}

\begin{wrapfigure}[18]{R}{0.525\columnwidth}
\vspace{-0.7cm}
    \hspace{4mm}\includegraphics[width=\linewidth]{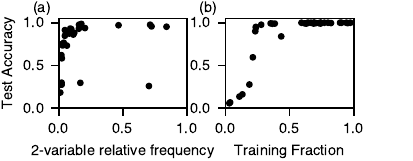}
    \caption{\textbf{Data requirements for  generalization.} \textbf{(a)} Test accuracy on 2-variable sequences as a function of their relative frequency in the training set. The model requires a relative frequency of $r \approx 0.2$ to successfully generalize. \textbf{(b)} Test accuracy on 0-variable sequences as a function of the fraction of all possible addition pairs seen during training. Generalization to unseen constant pairs requires training on at least $f \approx 0.25$ of the total distribution.}
    \label{fig:generalization_combined}
\end{wrapfigure}

In Figure~\ref{fig:accuracies}, our model achieves over $98\%$ test accuracy across all addition types ($0$-var, $1$-var, and $2$-var). We find that learning proceeds in distinct phases: the model first learns to perform addition on sequences with no variable operands ($0$-var addition), and only later learns variable--constant and variable--variable addition. In Section~\ref{sec:training}, we analyze this behavior in more detail. Notably, the model learns to generalize to sequences with novel variable positions. For example, after learning from \texttt{b=17} and \texttt{14+b=31}, the model can correctly predict the previously unseen ordering \texttt{b+14} and produce \texttt{31}. On a curated test set consisting specifically of such cases, the model achieves $99.6\%$ accuracy. Moreover, although training omits specific modular additions---for example, a dataset may contain no instance of $1+3$ in any form, including \texttt{1+3}, \texttt{c=1, c+3}, and \texttt{c=3, 1+c}---the model nevertheless learns to predict these held-out cases.  Similarly, on a curated test set consisting specifically of such latter cases, the model achieves $99.4\%$ accuracy.

\paragraph{Training set pairs.}
We train multiple models using different fractions $f \in (0,1)$ of the $59^2$ addition pairs in the training set. Figure~\ref{fig:generalization_combined} shows the zero-variable test accuracy after $30000$ training steps as a function of the fraction of addition pairs included in training. We observe we need at least $f \gtrsim 0.25$ in order to generalize to the entire $0$-var set.

\vspace{-0.1cm}
\paragraph{Two-variable priming} We train multiple models with different relative frequencies $r \in (0,1]$ as the ratio of the number of two-variable addition sequences to the number of zero-variable sequences in the training set. Figure~\ref{fig:generalization_combined} displays the two-variable test accuracy after $30000$ training steps. We observe $r \gtrsim 0.2$ is required for the model to generalize to all $2$-variable sequences.

\vspace{-0.2cm}

\section{Emergent Mechanisms and Compositionality}
\label{sec:mechanisms}
\vspace{-0.2cm}

\begin{figure}[t]
    \centering    
    \includegraphics[width=0.85\columnwidth]{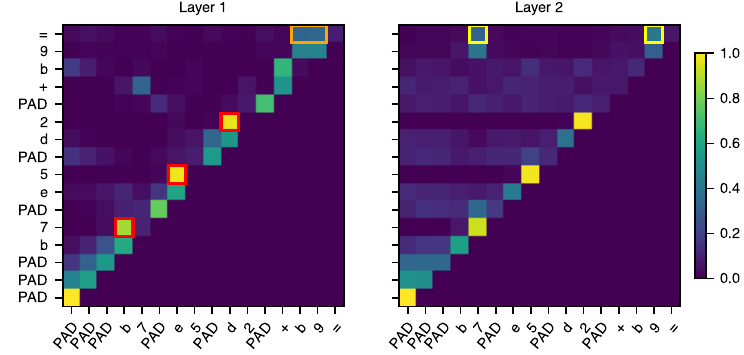}
    \caption{\textbf{Attention patterns for an example sequence.} \textbf{Left (Layer 1):} The \texttt{=} token attends to the two immediately preceding positions representing the operands (orange boxes). Simultaneously, constant tokens in positions 1--11 act as previous-token heads, attending to their assigned variables (red boxes). \textbf{Right (Layer 2):} The \texttt{=} token attends directly to the constants required for the addition (yellow boxes). These behaviors are consistent across all sequences and training initializations.}
    \label{fig:attention}
\vspace{-4mm}
\end{figure}
In this section, we analyze the model's internals at the end of training to empirically derive the mechanism driving generalization via composition.

We first analyze the attention patterns for an exemplary sequence in Figure~\ref{fig:attention}. In the second attention layer, the final \texttt{=} token—the position responsible for outputting the result—attends strongly to the specific operands required for the computation (e.g., \texttt{7} and \texttt{9} in the figure). This targeted attention implies that the modular addition itself is computed downstream of the query-key circuit, specifically via the combination of the layer-2 output-value (OV) matrix and the layer-2 MLP. 

Next, we isolate the OV-MLP circuit from the query-key routing to demonstrate how these components decouple in order to compute modular addition with or without variables.

\subsection{Modular addition with the OV-MLP Circuit}

\begin{wrapfigure}[20]{R}{0.525\columnwidth}
 \vspace{-0.05cm}  
    \hspace{4mm}\includegraphics[width=\linewidth]{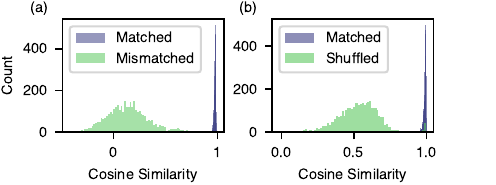}
    \caption{\textbf{Residual stream similarities.} \textbf{(a)} Cosine similarities between the pre-MLP residual stream vectors of different sequences. "Matched" pairs contain the same underlying addition operation (e.g., \texttt{b 3 + b 4 =} versus \texttt{+ 3 4 =}), whereas "Mismatched" pairs do not. The high similarity suggests that a shared representation handles both variable and constant formats. \textbf{(b)} Cosine similarity between the pre-MLP residual vector and the sum of the layer-2 output-value projections of the target operands, $OV_2(a) + OV_2(b)$. The high similarity compared to the randomly "Shuffled" baseline suggests that the modular addition operation is isolated to the $QK_2$-MLP circuit.}
    \label{fig:cosine_sim_combined}
\end{wrapfigure}

To confirm the above intuition that the modular addition circuit happens downstream of the layer $2$ query-key matrix, we analyze the residual stream vector at the \texttt{=} token immediately preceding the layer-2 MLP. Let $a$ and $b$ denote the raw token embeddings of the two target constants that are being added. In Figure~\ref{fig:cosine_sim_combined} we compute the cosine similarity between the pre-MLP residual vector and the sum of their independent layer-2 output-value projections, $OV_2(a) + OV_2(b)$. To establish a baseline, we contrast this matched similarity against the similarity computed using randomly shuffled operand embeddings. We find that the matched similarity is very close to 1 and much higher than the randomly shuffled cosine similarity, suggesting that the MLP essentially takes the vector $OV_2(a) + OV_2(b)$ as its input, regardless of the specific sequence.

\paragraph{OV-MLP performs modular addition.} To verify the MLP's active involvement in the addition circuit, we evaluate linear probes trained on the model's internal activations. While probes applied prior to the layer-2 MLP fail to predict the model's output, post-MLP probes succeed. This contrast demonstrates that indeed the layer-2 MLP is critical to compute the modular addition. As a final validation, we evaluate the isolated circuit directly. Computing the exact subnetwork $W_U(\mathrm{MLP}_2(OV_2(a) + OV_2(b)))$ on all pairs of constants achieves a $99.94\%$ accuracy, nearly perfectly recovering the full model's performance on the modular addition dataset.

\paragraph{Variable-number equivalence.} To verify that the model employs an identical circuit for both variable-assigned and pure constant addition, we analyze the residual stream vector immediately preceding the final layer MLP. We compute the cosine similarity between sequences that share the same underlying addition operation (e.g., \texttt{b = 3 + b 4 =} versus \texttt{+ 3 4 =}) and compare this against sequences with mismatched operands. As shown in Figure~\ref{fig:cosine_sim_combined}a, the cosine similarity is very close to 1 and significantly higher for matched operand pairs than for mismatched pairs. This confirms that a shared computational mechanism handles both input formats.

\begin{wrapfigure}[24]{R}{0.59\columnwidth}
\vspace{-0.8cm}
    \raggedright
    \hspace{2mm}\includegraphics[width=\linewidth]{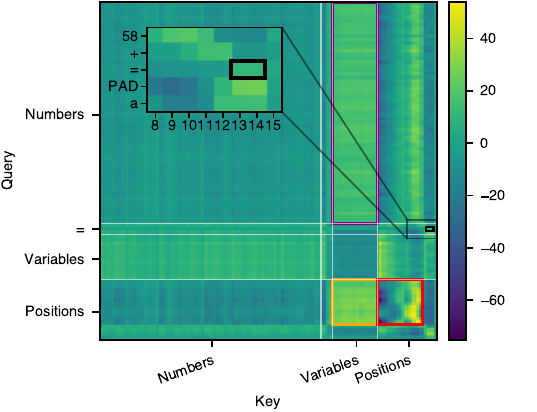}
    \caption{\textbf{Layer 1 QK matrix.} The previous-token head is formed by a combination of effects: constant tokens exhibit a strong preference for variables (purple outline); sequence positions 0--11, where variable assignments occur, also prefer variables (orange outline); and the positional interaction term generally increases toward the right (red outline). Furthermore, the \texttt{=} token attends specifically to the two preceding operand positions involved in the final addition (black outline).}
    \label{fig:qk0}
    \vspace{-2mm}
\end{wrapfigure}

\paragraph{Emergence of Fourier features.} To determine whether our model converges to known modular addition circuitry~\citep{gromov2023grokking}, we analyze the neuron preactivations in the layer-$2$ MLP. Specifically, we test for the presence of Fourier features, predicting that the preactivations for operands $n$ and $m$ take the theoretical form $h_k(n,m)=\cos\left(\frac{2\pi k}{N}n + \alpha_k \right) + \cos\left(\frac{2\pi k}{N}m + \beta_k \right)$. We observe a near-perfect match between this theoretical prediction and our empirically learned preactivations (see Figure~\ref{fig:preacts} in Appendix~\ref{appx:supplementary}), confirming the model relies on this established periodic mechanism.

\subsection{Variable Assignment via Induction Head} \label{variable_assignment}

In this section, we reverse engineer the variable assignment module of our transformer. We begin by examining the attention patterns in layer 1, shown in Figure~\ref{fig:attention}. These patterns exhibit two key traits: the \texttt{=} token attends specifically to positions 13 and 14, and the constant tokens preceding the \texttt{+} token attend to the variable situated directly before them.

To understand these behaviors, we analyze the layer 1 $QK$ matrix by testing it against all pairs of embeddings (Figure~\ref{fig:qk0}). We observe three distinct effects. First, the \texttt{=} token indeed exhibits a preference for the positional embeddings at 13 and 14 (black inset). Second, constant tokens interact strongly with variable tokens (purple outline), and positions 0--11---where the variable assignments occur---similarly attend to variables (orange outline). Third, the position-position attention term (red outline) generally increases towards the right.

Figure~\ref{fig:qk0pos_and_qk1}a analyzes this positional behavior in greater detail. Given the established preference for variables, we isolate the valid positions where a variable token can logically appear relative to a constant at position $i$. Specifically, we exclude positions greater than $i$ (future tokens) due to causal masking and position $i-2$, since the structural syntax dictates that a token two steps behind a constant must be either a \texttt{PAD} or another constant token. With these restrictions applied, Figure~\ref{fig:qk0pos_and_qk1} reveals that constants in positions 0--11 possess a strict positional preference for the token immediately preceding them. This dynamic effectively acts as a previous-token head, ensuring constants attend directly to their assigned variables.

Consequently, after layer 1, the \texttt{=} token contains remapped embeddings~$OV_1(e_{\text{var}})$ of the variables in the addition in its residual stream, while the constant tokens in the assignments contained the remapped embeddings of the corresponding variables. Then, in the second layer attention, as Figure~\ref{fig:qk0pos_and_qk1}b demonstrates, the $QK_2$ score between $OV_1(e_{\text{var}_1})$ and $OV_1(e_{\text{var}_2})$ is maximized strictly when $\text{var}_1 = \text{var}_2$. This confirms that the \texttt{=} token correctly attends to the specific constants associated with the variables used in the addition. By feeding these correct operands downstream, the attention mechanism directly enables the MLP to compute the final modular addition, successfully composing the two distinct operations.

\begin{figure}[ht]
    \centering
    \includegraphics[width=0.85\columnwidth]{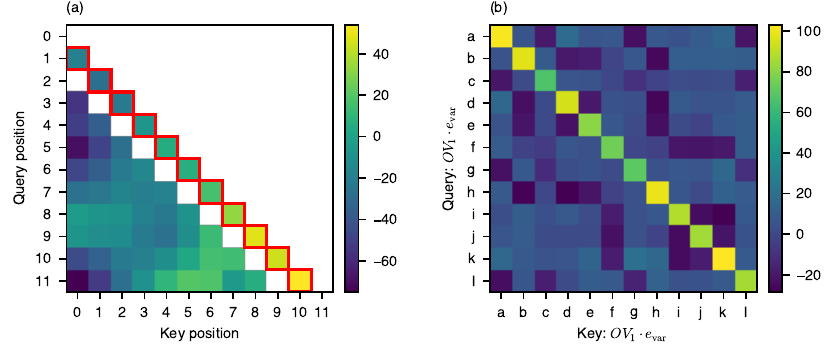}
    \caption{ \textbf{(a) Layer 1 QK matrix on positional embeddings:} Causal masking is applied. The diagonals $j=i$ and $j=i-2$ are excluded because a variable cannot appear in those positions relative to a constant at $i$. Red boxes indicate the maximum score in each row, demonstrating a strict preference for the immediately preceding position ($j=i-1$). \textbf{(b) Variable embedding similarity in Layer 2:} The transformation $(OV_1 E) QK_2 (OV_1 E)^\top$ exhibits a strong identity pattern. This demonstrates that tokens holding the same $OV_1(e_{\text{var}})$ embedding will attend to one another.}
    \label{fig:qk0pos_and_qk1}
    \vspace{-4mm}
\end{figure}

\section{Analyzing Training Dynamics with Progress Measures} \label{sec:training}
In this section, we further analyze the different phases of training dynamics observed in Figure~\ref{fig:accuracies},  by inspecting models at intermediate steps and implementing progress measures on specialized test sets in order to obtain a more fine-grained understanding. 
Recall that in the first phase (the first 5\,000 steps) the modular addition circuit is learned quickly before variable binding, although the latter begins to slowly emerge, as can be seen from the progress measures in Appendix~\ref{appx:progress} Figure~\ref{fig:progress_measures}. We now investigate the accuracy ``spikes'' around steps 5000 and 15000 in Figure~\ref{fig:accuracies}.

The first spike occurs between steps 5300 and 5400 during the second phase of training, when the accuracy of the model on all non-var-restricted evaluation sets increases rapidly. Figure~\ref{fig:early} shows that this transition coincides with the emergence of the variable-assignment routing circuit. Before the spike, the layer-1 attention pattern from the final \texttt{=} token still attends only to constant operands in 1-variable sequences, and the $QK_2$ matrix is only weakly diagonal on variable embeddings. After the spike, the \texttt{=} token attends cleanly to the two operand positions, while the matrix $(OV_1(e_{\mathrm{var}}))QK_2(OV_1(e_{\mathrm{var}}))^\top$ develops the diagonal structure discussed in Section \ref{variable_assignment}, producing the sharp increase in variable-containing accuracies.

The second spike, occurring between steps 14500 and 16000 in the final ``refinement'' phase of training, is more targeted: it brings the accuracy on var-restricted test sets in line with the rest. The failure of the model to generalize to var-restricted sequences before step 14500 is due to the model's failure to correctly route sequences with \texttt{b} as the second operand, which are held out from the training set. This has two connected causes: the \texttt{=} token attends less strongly to \texttt{b} than it does to other variables in the first layer (Figure~\ref{fig:late}c), and the \texttt{=} token fails to attend to the variable-assigned constant in layer 2 when $\texttt{b}$ is the second operand in a sequence (Figure~\ref{fig:late}d). As the $QK_1$ scores for \texttt{b} improve, the $QK_2$ scores and the accuracies for sequences with \texttt{b} as the second operand similarly improve.

\begin{figure}[t]
    \centering    \includegraphics[width=0.8\columnwidth]{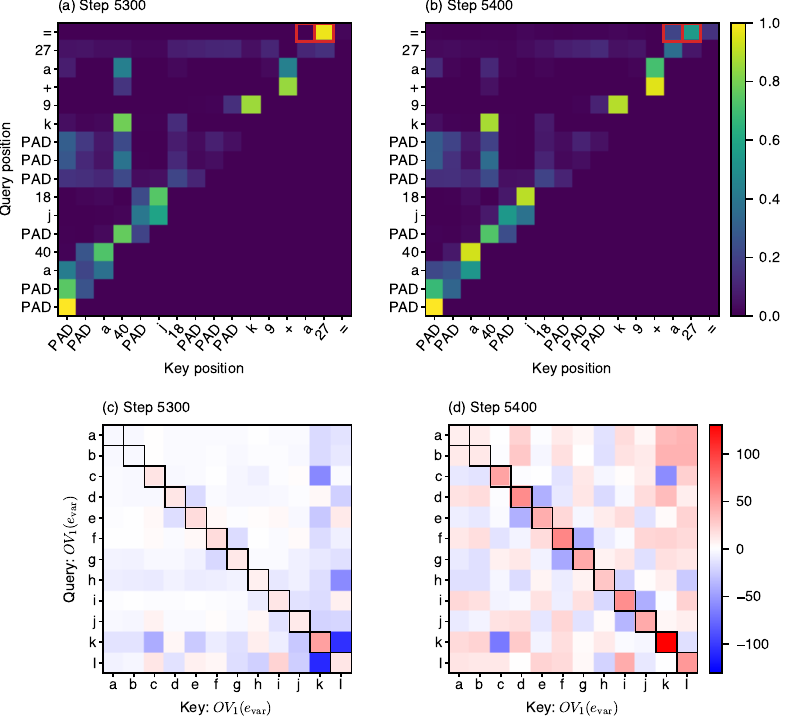}
    \caption{\textbf{Early emergence of variable assignment.} \textbf{(a,b)} Layer-1 attention patterns on a fixed 1-variable sequence immediately before and after the first accuracy spike. Red boxes mark the two operand positions that are queried by the \texttt{=} token in the final model. \textbf{(c,d)} The corresponding layer-2 QK scores on variables, $(OV_1(e_{\mathrm{var}}))QK_2(OV_1(e_{\mathrm{var}}))^\top$. Across this transition, the variable-identity block becomes strongly diagonal, allowing the model to retrieve assigned constants.}
    \label{fig:early}
    \vspace{-4mm}
\end{figure}

\begin{figure}[h]
    \centering
    \includegraphics[width=0.8\columnwidth]{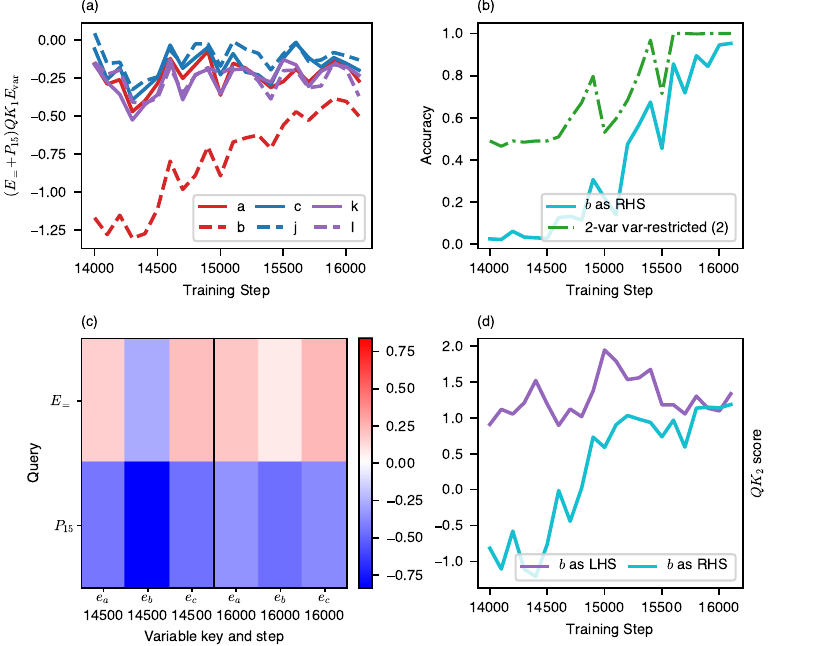}
    \caption{\textbf{Late correction of var-restricted routing.} \textbf{(a)} Variable-token contributions to the layer-1 attention from the \texttt{=} token, shown for selected variables. The \texttt{b} contribution begins substantially below the others and rises by the end of the window. \textbf{(b)} Accuracy on two-variable examples with \texttt{b} as the right operand, together with the fully var-restricted two-variable accuracy. The model completely fails to handle sequences with \texttt{b} as the right operand at the beginning of the window. \textbf{(c)} The projected variable-key matrix $[E_=;P_{15}]QK_1E_{\mathrm{var}}^\top$ for $e_a,e_b,e_c$ at steps 14500 and 16000. \textbf{(d)} A component of the $QK_2$ score between the tokens \texttt{=} and 7, $\sum_{p=13}^{14}a^{(1)}_{=,p}a^{(1)}_{4,3}((E_{x_p}+E_p)OV_1)QK_2(E_bOV_1)^\top$ for the sequences \texttt{PAD PAD b 7 $\cdots$ + b/9 9/b =} where $a^{(1)}_{=,p}$ denotes the $QK_1$ score from \texttt{=} to token at position $p$ (either \texttt{b} or 9, and $a^{(1)}_{4,3}$ denotes the $QK_1$ score from 7 to \texttt{b}.} 
    \label{fig:late}
    \vspace{-4mm}
\end{figure}

\section{Theoretical Insights on Training Dynamics}

To further shed light on how the transformer learns to compose variable binding and modular addition, we now provide theoretical insights on how gradient dynamics can lead to such emergent compositionality, in a simplified model.
We focus on the training of the second layer attention on one-variable data after the model has already learned the zero-variable modular addition mechanism, since the latter has already been thoroughly studied in previous works~\cite[e.g.,][]{morwani2023feature,tian2025composingglobalsolutionsreasoning,kunin2026alternating,he2026mechanism}.

Specifically, we consider a ``disentangled'' transformer~\cite{friedman2023learning,bietti2023birth,nichani2024transformers} setup where all embeddings are normal and value-output matrices remap embeddings to new orthogonal subspaces. We denote by~$e_n$ for $n \in \mathcal N$ the number embeddings, $e_v$ for $v \in \mathcal V$ the variable embeddings, and~$p_t$ for $t \in [T]$ the positional embeddings, and consider ``merged'' key-query and output-value matrices $W_{KQ}^\ell$ and~$W_{OV}^\ell$ at layer~$\ell = 1, 2$. We assume the first layer attention as well as the modular addition MLP are already in place, while the second layer key-query matrix initially only attends to operand positions. We show that training this key-query matrix on one-variable data then leads to the relevant associative memory block which enables composition of variable assignment with modular addition.

\begin{theorem}[Informal]
    In a simplified setup, the negative 1-variable population gradient w.r.t.~$W_{KQ}^2$ at~$\bar W_{KQ}^2 = \beta (p_{T-2} + p_{T-1}) p_T^\top + \beta \sum_{n \in \mathcal N} e_n p_T^\top$, with large enough~$\beta$, has the following associative memory as its dominant term when~$N, T \gg V$:
    $\sum_{v \in \mathcal V} W_{OV}^1 e_v (W_{OV}^1 e_v + p_T)^\top$.
\end{theorem}

Details are provided in Appendix~\ref{appx:theory}.
Taking an update along this negative gradient direction with appropriate step-size then leads to the $QK_2$ mechanism studied in Section~\ref{sec:mechanisms}, which enables composition of variable binding with modular addition.
The initialization~$\bar W_{KQ}^2$ with large enough~$\beta$ ensures an initialization where the second layer attention focuses on the constants provided as operands, consistent with our experimental findings for the zero-variable mechanism. Our analysis shows that the key-query gradient is dominated by terms involving the token with the correct assignment constant, thanks to a constructive interaction with the gradient of the MLP for the correct label.

\section{Conclusion}
In this work, we devise a mechanistically interpretable model for compositional generalization by considering a small transformer trained on a task that combines variable assignment and modular arithmetic. We demonstrate that the emergence of out-of-distribution generalization in training occurs via composition of independently-learned modular addition and variable assignment modules. The variable assignment is learned via a modified induction head circuit which we mathematically analyze, while modular addition follows the known Fourier circuit in the second-layer MLP. These results provide a mechanistic basis for understanding compositional generalization in more general mathematical reasoning problems.

Although our study provides several insights into compositional generalization through mechanistic analysis in a controlled setting, we only consider a task and model of limited complexity in order to ease analysis. We hope it can nonetheless motivate future investigations that may consider richer and more complex tasks, as well as more detailed and quantitative analysis which could lead further insight, e.g., on sample complexity or scaling considerations.

\bibliography{neurips}

@inproceedings{Vaswani+2017,
  author       = {Ashish Vaswani and
                  Noam Shazeer and
                  Niki Parmar and
                  Jakob Uszkoreit and
                  Llion Jones and
                  Aidan N. Gomez and
                  Lukasz Kaiser and
                  Illia Polosukhin},
  title        = {Attention is All you Need},
  booktitle    = {Advances in Neural Information Processing Systems (NeurIPS)},
  year         = {2017},
}

@inproceedings{wang2022interpretability,
  author       = {Kevin Ro Wang and
                  Alexandre Variengien and
                  Arthur Conmy and
                  Buck Shlegeris and
                  Jacob Steinhardt},
  title        = {Interpretability in the Wild: a Circuit for Indirect Object Identification
                  in {GPT-2} Small},
  booktitle    = {International Conference on Learning Representations (ICLR)},
  year         = {2023}
}

@inproceedings{hanna2023does,
  author       = {Michael Hanna and
                  Ollie Liu and
                  Alexandre Variengien},
  title        = {How does {GPT-2} compute greater-than?: Interpreting mathematical
                  abilities in a pre-trained language model},
  booktitle    = {Advances in Neural Information Processing Systems (NeurIPS)},
  year         = {2023}
}

@article{olsson2022context,
    title = {In-context Learning and Induction Heads},
    author = {Olsson, Catherine and Elhage, Nelson and Nanda, Neel and
              Joseph, Nicholas and DasSarma, Nova and Henighan, Tom and Mann,
              Ben and Askell, Amanda and Bai, Yuntao and Chen, Anna and
              Conerly, Tom and Drain, Dawn and Ganguli, Deep and
              Hatfield-Dodds, Zac and Hernandez, Danny and Johnston, Scott and
              Jones, Andy and Kernion, Jackson and Lovitt, Liane and Ndousse,
              Kamal and Amodei, Dario and Brown, Tom and Clark, Jack and
              Kaplan, Jared and McCandlish, Sam and Olah, Chris},
    journal = {Transformer Circuits Thread},
    year = {2022},
}

@inproceedings{vonoswald2023transformers,
  author       = {Johannes von Oswald and
                  Eyvind Niklasson and
                  Ettore Randazzo and
                  Jo{\~{a}}o Sacramento and
                  Alexander Mordvintsev and
                  Andrey Zhmoginov and
                  Max Vladymyrov},
  title        = {Transformers Learn In-Context by Gradient Descent},
  booktitle    = {International Conference on Machine Learning, (ICML)},
  year         = {2023},
  bibsource    = {dblp computer science bibliography, https://dblp.org}
}

@inproceedings{bietti2023birth,
  author       = {Alberto Bietti and
                  Vivien Cabannes and
                  Diane Bouchacourt and
                  Herv{\'{e}} J{\'{e}}gou and
                  L{\'{e}}on Bottou},
  title        = {Birth of a Transformer: {A} Memory Viewpoint},
  booktitle    = {Advances in Neural Information Processing Systems (NeurIPS)},
  year         = {2023},
  bibsource    = {dblp computer science bibliography, https://dblp.org}
}

@misc{power2022grokking,
      title={Grokking: Generalization Beyond Overfitting on Small Algorithmic Datasets}, 
      author={Alethea Power and Yuri Burda and Harri Edwards and Igor Babuschkin and Vedant Misra},
      year={2022},
      eprint={2201.02177},
      archivePrefix={arXiv},
      primaryClass={cs.LG},
}

@inproceedings{nanda2023progress,
  author       = {Neel Nanda and
                  Lawrence Chan and
                  Tom Lieberum and
                  Jess Smith and
                  Jacob Steinhardt},
  title        = {Progress measures for grokking via mechanistic interpretability},
  booktitle    = {International Conference on Learning Representations (ICLR)},
  year         = {2023},
  bibsource    = {dblp computer science bibliography, https://dblp.org}
}

@misc{gromov2023grokking,
      title={Grokking modular arithmetic}, 
      author={Andrey Gromov},
      year={2023},
      eprint={2301.02679},
      archivePrefix={arXiv},
      primaryClass={cs.LG},
}

@techreport{radford2019language,
    title = {Language Models are Unsupervised Multitask Learners},
    author = {Radford, Alec and Wu, Jeffrey and Child, Rewon and Luan, David and
              Amodei, Dario and Sutskever, Ilya},
    institution = {OpenAI},
    year = {2019},
}

@misc{nanda2022transformerlens,
    title = {TransformerLens},
    author = {Neel Nanda and Joseph Bloom},
    year = {2022},
    howpublished = {\url{https://github.com/TransformerLensOrg/TransformerLens}},
}

@inproceedings{cabannes2024iteration,
  author       = {Vivien Cabannes and
                  Charles Arnal and
                  Wassim Bouaziz and
                  Xingyu Yang and
                  Fran{\c{c}}ois Charton and
                  Julia Kempe},
  title        = {Iteration Head: {A} Mechanistic Study of Chain-of-Thought},
  booktitle    = {Advances in Neural Information Processing Systems (NeurIPS)},
  year         = {2024},
  bibsource    = {dblp computer science bibliography, https://dblp.org}
}

@misc{zhao2026shatteredcompositionalitycounterintuitivelearning,
      title={Shattered Compositionality: Counterintuitive Learning Dynamics of Transformers for Arithmetic}, 
      author={Xingyu Zhao and Darsh Sharma and Rheeya Uppaal and Yiqiao Zhong},
      year={2026},
      eprint={2601.22510},
      archivePrefix={arXiv},
      primaryClass={cs.LG},
}

@inproceedings{loshchilov2018decoupled,
  author       = {Ilya Loshchilov and
                  Frank Hutter},
  title        = {Decoupled Weight Decay Regularization},
  booktitle    = {International Conference on Learning Representations (ICLR)},
  year         = {2019},
  bibsource    = {dblp computer science bibliography, https://dblp.org}
}

@article{elhage2021mathematical,
   title={A Mathematical Framework for Transformer Circuits},
   author={Elhage, Nelson and Nanda, Neel and Olsson, Catherine and Henighan, Tom and Joseph, Nicholas and Mann, Ben and Askell, Amanda and Bai, Yuntao and Chen, Anna and Conerly, Tom and DasSarma, Nova and Drain, Dawn and Ganguli, Deep and Hatfield-Dodds, Zac and Hernandez, Danny and Jones, Andy and Kernion, Jackson and Lovitt, Liane and Ndousse, Kamal and Amodei, Dario and Brown, Tom and Clark, Jack and Kaplan, Jared and McCandlish, Sam and Olah, Chris},
   year={2021},
   journal={Transformer Circuits Thread},
}

@inproceedings{morwani2023feature,
  author       = {Depen Morwani and
                  Benjamin L. Edelman and
                  Costin{-}Andrei Oncescu and
                  Rosie Zhao and
                  Sham M. Kakade},
  title        = {Feature emergence via margin maximization: case studies in algebraic
                  tasks},
  booktitle    = {International Conference on Learning Representations (ICLR)},
  year         = {2024},
  bibsource    = {dblp computer science bibliography, https://dblp.org}
}

@misc{tian2025composingglobalsolutionsreasoning,
      title={Composing Global Solutions to Reasoning Tasks via Algebraic Objects in Neural Nets}, 
      author={Yuandong Tian},
      year={2025},
      eprint={2410.01779},
      archivePrefix={arXiv},
      primaryClass={cs.LG},
}

@inproceedings{
kunin2026alternating,
title={Alternating Gradient Flows: A Theory of Feature Learning in Two-layer Neural Networks},
author={Daniel Kunin and Giovanni Luca Marchetti and Feng Chen and Dhruva Karkada and James B Simon and Michael R DeWeese and Surya Ganguli and Nina Miolane},
booktitle={Advances in Neural Information Processing Systems (NeurIPS)},
year={2026},
}

@misc{he2026mechanism,
      title={On the Mechanism and Dynamics of Modular Addition: Fourier Features, Lottery Ticket, and Grokking}, 
      author={Jianliang He and Leda Wang and Siyu Chen and Zhuoran Yang},
      year={2026},
      eprint={2602.16849},
      archivePrefix={arXiv},
      primaryClass={cs.LG},
}

@inproceedings{nichani2024transformers,
  author       = {Eshaan Nichani and
                  Alex Damian and
                  Jason D. Lee},
  title        = {How Transformers Learn Causal Structure with Gradient Descent},
  booktitle    = {International Conference on Machine Learning (ICML)},
  year         = {2024},
  bibsource    = {dblp computer science bibliography, https://dblp.org}
}

@inproceedings{friedman2023learning,
  author       = {Dan Friedman and
                  Alexander Wettig and
                  Danqi Chen},
  title        = {Learning Transformer Programs},
  booktitle    = {Advances in Neural Information Processing Systems (NeurIPS)},
  year         = {2023},
  bibsource    = {dblp computer science bibliography, https://dblp.org}
}

@inproceedings{press2023compositionalitygap,
  author       = {Ofir Press and
                  Muru Zhang and
                  Sewon Min and
                  Ludwig Schmidt and
                  Noah A. Smith and
                  Mike Lewis},
  title        = {Measuring and Narrowing the Compositionality Gap in Language Models},
  booktitle    = {Findings of the Association for Computational Linguistics (EMNLP)},
  year         = {2023},
  bibsource    = {dblp computer science bibliography, https://dblp.org}
}

@inproceedings{ramesh2024compositionalcapabilities,
  author       = {Rahul Ramesh and
                  Ekdeep Singh Lubana and
                  Mikail Khona and
                  Robert P. Dick and
                  Hidenori Tanaka},
  title        = {Compositional Capabilities of Autoregressive Transformers: {A} Study
                  on Synthetic, Interpretable Tasks},
  booktitle    = {International Conference on Machine Learning (ICML)},
  year         = {2024},
  bibsource    = {dblp computer science bibliography, https://dblp.org}
}

@inproceedings{wang2025easytohard,
  author       = {Zixuan Wang and
                  Eshaan Nichani and
                  Alberto Bietti and
                  Alex Damian and
                  Daniel Hsu and
                  Jason D. Lee and
                  Denny Wu},
  title        = {Learning Compositional Functions with Transformers from Easy-to-Hard
                  Data},
  booktitle    = {Annual Conference on Learning Theory (COLT)},
  year         = {2025},
  bibsource    = {dblp computer science bibliography, https://dblp.org}
}

@misc{davies2023variablebinding,
  author       = {Xander Davies and Max Nadeau and Nikhil Prakash and Tamar Rott Shaham and David Bau},
  title        = {Discovering Variable Binding Circuitry with Desiderata},
  year         = {2023},
  howpublished = {ICML 2023 Workshop on Deployable Generative AI}
}

@inproceedings{lake2018scan,
  author       = {Brenden M. Lake and
                  Marco Baroni},
  title        = {Generalization without Systematicity: On the Compositional Skills
                  of Sequence-to-Sequence Recurrent Networks},
  booktitle    = {International Conference on Machine Learning (ICML)},
  year         = {2018},
  bibsource    = {dblp computer science bibliography, https://dblp.org}
}

@inproceedings{kim2020cogs,
  author       = {Najoung Kim and
                  Tal Linzen},
  title        = {{COGS:} {A} Compositional Generalization Challenge Based on Semantic
                  Interpretation},
  booktitle    = {Conference on Empirical Methods in Natural Language Processing (EMNLP)},
  year         = {2020},
  bibsource    = {dblp computer science bibliography, https://dblp.org}
}

@inproceedings{he2024learningtogrok,
  author       = {Tianyu He and
                  Darshil Doshi and
                  Aritra Das and
                  Andrey Gromov},
  title        = {Learning to grok: Emergence of in-context learning and skill composition
                  in modular arithmetic tasks},
  booktitle    = {Advances in Neural Information Processing Systems (NeurIPS)},
  year         = {2024},
}

@inproceedings{dziri2023faith,
  author       = {Nouha Dziri and
                  Ximing Lu and
                  Melanie Sclar and
                  Xiang Lorraine Li and
                  Liwei Jiang and
                  Bill Yuchen Lin and
                  Sean Welleck and
                  Peter West and
                  Chandra Bhagavatula and
                  Ronan Le Bras and
                  Jena D. Hwang and
                  Soumya Sanyal and
                  Xiang Ren and
                  Allyson Ettinger and
                  Za{\"{\i}}d Harchaoui and
                  Yejin Choi},
  title        = {Faith and Fate: Limits of Transformers on Compositionality},
  booktitle    = {Advances in Neural Information Processing Systems (NeurIPS)},
  year         = {2023},
  bibsource    = {dblp computer science bibliography, https://dblp.org}
}

@inproceedings{reddymechanistic,
  author       = {Gautam Reddy},
  title        = {The mechanistic basis of data dependence and abrupt learning in an
                  in-context classification task},
  booktitle    = {International Conference on Learning Representations (ICLR)},
  year         = {2024},
  bibsource    = {dblp computer science bibliography, https://dblp.org}
}

@inproceedings{dar2023analyzing,
  author       = {Guy Dar and
                  Mor Geva and
                  Ankit Gupta and
                  Jonathan Berant},
  title        = {Analyzing Transformers in Embedding Space},
  booktitle    = {Findings of the Association for Computational Linguistics (ACL)},
  year         = {2023},
  bibsource    = {dblp computer science bibliography, https://dblp.org}
}

@article{wu2025transformers,
  title={How Do Transformers Learn Variable Binding in Symbolic Programs?},
  author={Wu, Yiwei and Geiger, Atticus and Milliere, Rapha{\"e}l},
  journal={arXiv preprint arXiv:2505.20896},
  year={2025}
}

@article{zhang2022unveiling,
  title={Unveiling transformers with lego: a synthetic reasoning task},
  author={Zhang, Yi and Backurs, Arturs and Bubeck, S{\'e}bastien and Eldan, Ronen and Gunasekar, Suriya and Wagner, Tal},
  journal={arXiv preprint arXiv:2206.04301},
  year={2022}
}

@inproceedings{liu2023transformers,
  title={Transformers Learn Shortcuts to Automata},
  author={Liu, Bingbin and Ash, Jordan T and Goel, Surbhi and Krishnamurthy, Akshay and Zhang, Cyril},
  booktitle={International Conference on Learning Representations (ICLR)},
  year={2023}
}

@article{ameisen2025circuit,
  title={Circuit tracing: Revealing computational graphs in language models},
  author={Ameisen, Emmanuel and Lindsey, Jack and Pearce, Adam and Gurnee, Wes and Turner, Nicholas L and Chen, Brian and Citro, Craig and Abrahams, David and Carter, Shan and Hosmer, Basil and others},
  journal={Transformer Circuits Thread},
  volume={6},
  pages={16318--16352},
  year={2025}
}
\bibliographystyle{unsrtnat}

\appendix
\section{Additional Experimental Results} \label{appx:supplementary}
Figure~\ref{fig:accuracies2} shows the accuracies by evaluation set for another training run with the same parameters. The same ``spikes'' in non-var-restricted and var-restricted accuracies seen in Figure~\ref{fig:accuracies} are also displayed here, although at different times.

\begin{figure}[h]
    \centering
    \includegraphics[width=0.625\columnwidth]{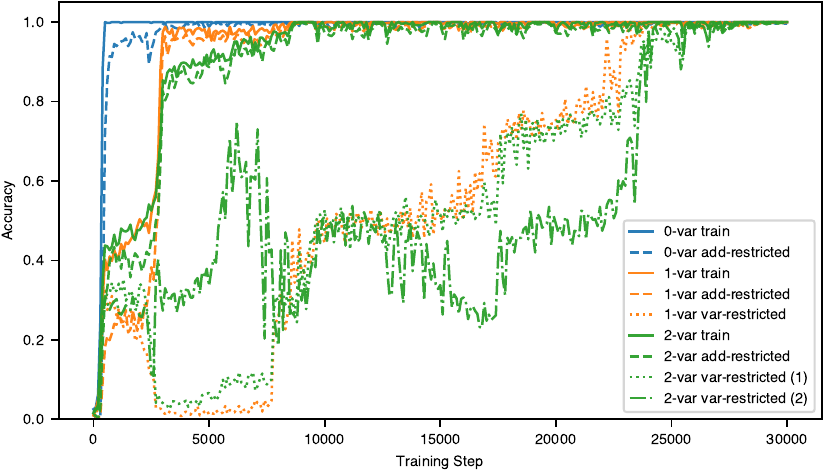}
    \caption{\textbf{Accuracies for alternate initialization.} For completeness, we recreate Figure~\ref{fig:accuracies} for another training run. In this training run, the lowered \texttt{b} accuracy discussed in Section~\ref{sec:training} instead appears for sequences with \texttt{k} on the left-hand side.}
    \label{fig:accuracies2}
\end{figure}

Figure~\ref{fig:preacts} displays the preactivations for neuron 70 (left) and the preactivations for the theoretical construction of an MLP trained to perform modular addition in \citet{gromov2023grokking}. The figure shows a near-perfect match, suggesting our model indeed implements the same MLP circuit that has been previously discovered in the literature.

\begin{figure}[h]
    \centering
    \includegraphics[width=0.85\columnwidth]{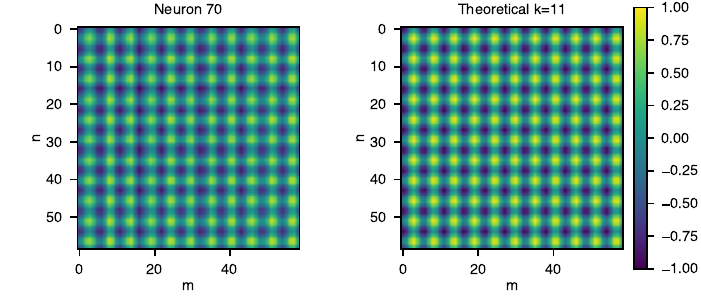}
    \caption{\textbf{Emergence of Fourier features in the layer-2 MLP.} The learned preactivations of a specific neuron (Neuron 70) across all operand pairs $(n, m)$. match the corresponding theoretical 2D Fourier feature for frequency $k=11$ predicted by \citet{gromov2023grokking}.}
    \label{fig:preacts}
\end{figure}

\section{Progress Measures} \label{appx:progress}
We implement a variety of progress measures to track training dynamics:
\begin{itemize}
    \item \textbf{ov2\_mlp2\_accuracy}: Tracks the accuracy of $W_U(MLP_2(OV_2(a)+OV_2(b)))$ over all $59^2$ addition pairs
    \item \textbf{qk1\_num\_to\_prev\_var}: A memory probe that tracks whether the $QK_1$ score for positional $\times$ positional embeddings is maximized when $j=i-1$.
    \item \textbf{qk2\_ov1\_var\_identity}: Tracks whether $(OV_1(E))QK_2(OV_1(E))^\top$ for variable embeddings is the identity.
    \item \textbf{attn\_l1\_equal\_to\_operands}: Tracks whether, for a batch of sequences, the \texttt{=} token attends to the operands of the addition in layer $1$.
    \item \textbf{attn\_l1\_num\_to\_prev}: Tracks whether, for a batch of sequences, the number tokens in positions $1-11$ attend to the previous (variable) token in layer $1$.
   \item \textbf{attn\_l2\_equal\_to\_values}: Tracks whether, for a batch of sequences, the \texttt{=} token attends to the correct constants in layer $1$.
    \item \textbf{probe\_l1\_var\_from\_num}: Tracks whether, for a batch of sequences, a linear probe can be trained to classify the corresponding variable name from the residual stream vector on the number tokens after the layer $1$ attention.
\end{itemize}
\begin{figure}
    \centering
    \includegraphics[width=0.75\columnwidth]{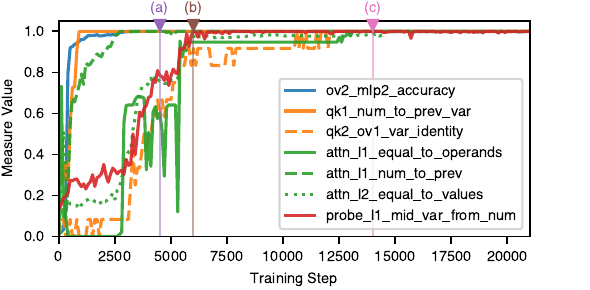}
    \caption{\textbf{(a)} End of the first phase of training, characterized by generalization on 0-variable addition and the emergence of structure for variable assignment. \textbf{(b)} End of the second phase of training, characterized by the rapid development of the variable assignment circuit \textbf{(c)} End of the last phase of training, characterized by the model ``cleaning up'' the variable assignment module and generalizing on all evaluation sets.}
    \label{fig:progress_measures}
\end{figure}
 
These progress measures reveal the order in which the model learns different elements of the circuit in training. In particular, training can be divided into three phases. In the first phase, \texttt{ov2\_mlp2\_accuracy} rises sharply, indicating that the final-layer MLP implements modular addition once the correct operands are provided. This coincides with the model achieving $100\%$ accuracy on $0$-variable sequences by the end. Furthermore, towards the end of this phase, we see the emergence of a variable binding mechanism with measures that track variable binding, such as \texttt{probe\_l1\_mid\_var\_from\_num}, reaching $80\%$.

In the second phase, the model rapidly learns to bind variables to their assigned constants, corresponding to the first ``spike'' discussed in Section~\ref{sec:training}. This is evidenced by \texttt{probe\_l0\_var\_from\_num} reaching $100\%$, indicating that variable names are successfully copied into the constant tokens' residual streams. 
However, the metrics tracking the routing between the assignment module and the addition module (\texttt{qk1\_ov0\_var\_identity}, \texttt{attn\_l0\_equal\_to\_operands}, and \texttt{attn\_l1\_equal\_to\_values}) plateau between 80\% and 99\%, corresponding to the model not yet being able to correct route constants assigned to \texttt{b} (Section~\ref{sec:training}).

In the final phase, the model bridges the remaining performance gap by , achieving $>99\%$ accuracy across all benchmarks, including the var-restricted sets. This ordering confirms that the model first masters the core modular addition operation, and subsequently learns to generalize by correctly routing the assigned variables into this established circuit.

\section{Theoretical Analysis of Gradient Dynamics}
\label{appx:theory}


\paragraph{Architecture.}
To simplify our analysis, we assume that the first layer mechanisms are already in place, by considering a single layer of attention on top of augmented inputs of the form:
\begin{align}
    x_t &= \begin{cases}
        p_t + e_{n_t} + \Phi_1 e_{v_t}, &\text{ if }t \in \tau \\
        p_t + e_{pad}, &\text{ if }t \in [T-3] \setminus \tau,
    \end{cases} \nonumber \\
    x_{T-2} &= p_{T-2} + e_{v^*} \label{eq:simplified_x}
\\
    x_{T-1} &= p_{T-1} + e_{m} \nonumber \\
    x_T &= p_T + \Phi_1 e_{v^*}. \nonumber
\end{align}
In this example, the operands are the number~$m \in \mathcal N$ and the variable~$v^* \in \mathcal V$, while the assignments are $(v_t, n_t) \in \mathcal V \times \mathcal N$ for~$t \in \tau \subset [T-3]$, meaning that they appear at positions~$t \in \tau$, while other positions are just pad tokens, and we assume~$v^* \in \{v_t\}_{t \in \tau}$.
The matrix~$\Phi_1$ mimicks the value matrix~$W_{OV}^1$, and we assume here that it maps embeddings to a new orthogonal subspace, similar to a random matrix in high dimensions, as in~\cite{bietti2023birth}.

The architecture we consider then is a single attention layer with an MLP on top:
\begin{equation}
\label{eq:simplified_F}
    F_\alpha(x_{1:T}) = \alpha M(\Phi_2 X \sigma(X^\top W x_T)),
\end{equation}
where~$X = [x_1 \cdots x_T] \in \mathbb R^{d \times T}$,~$W$ denotes the key-query matrix, $\Phi_2$ the value-output matrix, $\sigma$ the attention softmax, and~$M(\cdot)$ the MLP which maps $d$-dimensional representations~$x$ to unscaled logits~$M_k(x)$ for each vocabulary element~$k \in \mathcal N$.
The scalar~$\alpha > 0$ controls the initial scale of the final logits; the proof below considers the small-output regime~$\alpha \ll 1$, after which~$\alpha$ may be increased to sharpen predictions without changing the direction of the key-query signal analyzed here.
We assume~$\Phi_2$ is a fixed remapping to a new orthogonal subspace, and~$M$ follows the construction in~\cite[Appendix A]{gromov2023grokking} with complex-valued weights:
\begin{align*}
    &M_k(x) = \frac{1}{N} \sum_{j=1}^N u_{k,j} (v_j^\top x)^2, \text{ with } \\
     v_j = \sum_{n=1}^N &e^{2 i\pi \frac{j n}{N}} \Phi_2 e_n, \quad \text{ and  } \quad u_{k,j} = e^{-2 i \pi \frac{j k}{N}}.
\end{align*}
This implies~$M_k(\Phi_2 e_n + \Phi_2 e_m) = 2 \cdot \mathbf{1}\{k = n + m\} + \mathbf{1}\{k = 2 n \} + \mathbf{1}\{k = 2m\}$, where equalities are intended as modulo~$N$.
We also note that we have~$\nabla M_k(\Phi_2 e_n)^\top \Phi_2 e_m = 2 \cdot \mathbf{1}\{k = n + m\}$, a property which will be useful in our gradient analysis.

\paragraph{Gradient dynamics.}
We assume the model is initialized at
\begin{equation}
    \label{eq:init_w0}
    W_0 = \beta (p_{T-2} + p_{T-1}) p_T^\top + \beta \sum_{n \in \mathcal N} e_n p_T^\top,
\end{equation}
which leads to an attention behavior similar to that observed in Figure~\ref{fig:early}(a) when~$\beta$ is chosen appropriately, and analyze the behavior when taking gradient steps on one-variable data.

We consider a data distribution~$p$ over sequence-label pairs $(x_{1:T}, y)$ that are generated as follows: $x_{1:T}$ is a one-variable sequence of the form~\eqref{eq:simplified_x}, where we assume~$\tau = \{t^*\}$, with~$t^* \sim \text{Unif}([T-3])$, with an assignment variable~$v \sim \text{Unif}(\mathcal V)$, assigned number and second operand~$n, m \sim \text{Unif}(\mathcal N)$. The label is~$y = n + m \mod N$. The loss takes the form
    \begin{equation}
        \label{eq:population_loss}
        L(W) = \mathbb E_{(x, y) \sim p}[\ell(y, F_\alpha(x_{1:T}))],
    \end{equation}
    where~$\ell(y, \xi) = -\xi_y + \log \sum_k \exp (\xi_k)$ is the cross-entropy loss and~$F_\alpha$ is defined in~\eqref{eq:simplified_F}.

\begin{theorem}[One gradient step leads to composition] \label{thm:c1}
Consider the population loss~\eqref{eq:population_loss} over the data distribution~$p$ defined above, and the initialization~$W_0$ in~\eqref{eq:init_w0}.
Let~$\sigma = \sigma(X^\top W_0 x_T)$ and assume that~$\beta$ is large enough so that~$\sigma_{T-1} = \rho_1 = \Theta(1)$\footnote{Remark that~$\sigma_{T-1}$ is independent of the inputs for this specific data model considered, under orthogonal embeddings.} while~$\rho_2 := \sigma_{t^*} = \Theta(1/T)$, and~$\sigma_t = \Theta(1/T)$ for all~$t \ne T-1$.
Consider one gradient step~$W_1 = W_0 - \eta \nabla L(W_0)$ with learning rate
\[
    \eta = \frac{V N}{2 \alpha \rho_1 \rho_2 (N-1)}.
\]
Then
\begin{align*}
    W_1
    &= \tilde W_0 + \sum_{v \in \mathcal V} \Phi_1 e_v(\Phi_1 e_v + p_T)^\top
        + O_\infty\left(V\left(\frac{1}{N}+\frac{1}{T}+T\alpha\right)\right),
\end{align*}
where the~$O_\infty(\cdot)$ term is in matrix infinity norm in the orthogonal embedding basis, and where we define~$\tilde W_0 := W_0 - (2\rho_1 - 1)p_{T-1}( V p_T^\top + \sum_{v \in \mathcal V}(\Phi_1 e_v)^\top)$ as the initial matrix with an additional term, which attenuates the pre-existing attraction to position~$T\!-\!1$ when~$\rho_1 > 1/2$.
Note that when~$N,T \gg V$ and the final-logit scale~$\alpha$ is small enough that~$VT\alpha \ll 1$, the displayed associative-memory term dominates the update.
\end{theorem}
The associative-memory term in~$W_1$ is sufficient to ensure attention to the correct variables in the context, even when~$\tau$ has more than one variable. Note that the component~$\sum_v \Phi_1 e_v p_T^\top$ also gives some attention to all assignments, but with lower scores, and this spurious term should be attenuated if one takes subsequent steps on data involving more variables in~$\tau$.

The small-output scale condition above is a standard way of replacing the idealization~$\hat p(k|x_{1:T})=1/N$ by a perturbative argument: at small final-logit scale the prediction softmax remains close to uniform, and the non-uniformity contributes only a controlled remainder. Similar small-initialization arguments appear in theoretical analyses of transformer training dynamics~\cite[e.g.,][]{nichani2024transformers,wang2025easytohard}.

\begin{proof}

    Let~$\varepsilon_\alpha = e^{2\alpha} - 1$.
    The gradient at~$W = W_0$ is then given by
    \[
    \nabla L = \alpha \sum_{k \in \mathcal N} \mathbb E \left[(\hat p_\alpha(k | x_{1:T}) - \mathbf{1}\{y = k\}) \sum_t \sigma_t \nabla M_k(z)^\top \Phi_2 x_t \cdot (x_t - \bar x_{1:T}) x_T^\top \right],
    \]
    where we define~$\sigma_t = \sigma(X^\top W_0 x_T)_t$, $z = \Phi_2 X \sigma$, and~$\bar x_{1:T} = \sum_t \sigma_t x_t$, and~$\hat p_\alpha(k|x_{1:T}) = \mathrm{softmax}(\alpha M(z))_k$ denotes the prediction of the current model.
    We first bound the unscaled logits~$M(z)$.
    Under the orthogonality assumptions on~$\Phi_2$, the Fourier vectors~$v_j$ only overlap with the number-embedding components of the input, and are orthogonal to all other embeddings, so that we have~$|v_j^\top \Phi_2 x_t| \leq 1$, with equality only possible at positions with a number component.
    Therefore, for every~$j$,
    \[
        |v_j^\top z|
        =
        \left|\sum_t \sigma_t v_j^\top \Phi_2 x_t\right|
        \le \sum_t \sigma_t = 1,
    \]
    and, since~$|u_{k,j}|=1$,
    \[
        |M_k(z)|
        \le \frac{1}{N}\sum_{j=1}^N |v_j^\top z|^2
        \le 1.
    \]
    Hence, for every~$k$ we have
    \begin{align*}
        \frac{e^{-2\alpha}}{N}
        \le \hat p_\alpha(k|x_{1:T})
        \le \frac{e^{2\alpha}}{N}, \qquad
        \hat p_\alpha(k|x_{1:T}) &= \frac{1}{N} + r_k(x_{1:T}), \\
        |r_k(x_{1:T})| &\le \frac{\varepsilon_\alpha}{N}.
    \end{align*}
    Thus the gradient can be decomposed as~$\nabla L = \alpha(\nabla L_{\mathrm{unif}} + R_\alpha)$, where
    \begin{align*}
        \nabla L_{\mathrm{unif}}
        &=
        \sum_{k \in \mathcal N} \mathbb E \left[\left(\frac{1}{N} - \mathbf{1}\{y = k\}\right) A_k(x_{1:T})\right], \\
        R_\alpha
        &=
        \sum_{k \in \mathcal N} \mathbb E \left[r_k(x_{1:T}) A_k(x_{1:T})\right],
    \end{align*}
    with
    \[
        A_k(x_{1:T})
        =
        \sum_t \sigma_t \nabla M_k(z)^\top \Phi_2 x_t \cdot (x_t - \bar x_{1:T}) x_T^\top .
    \]
    The matrices~$A_k(x_{1:T})$ have uniformly~$O(1)$ entries: indeed, the bounds above imply
    \[
        |\nabla M_k(z)^\top \Phi_2 x_t|
        =
        \left|\frac{2}{N} \sum_j u_{k,j} (v_j^\top z)(v_j^\top \Phi_2 x_t)\right|
        \le 2,
    \]
    while~$\sum_t \sigma_t=1$, and the vectors~$x_t$, $\bar x_{1:T}$, and~$x_T$ have~$O(1)$ coordinates in the orthogonal embedding basis.
    Hence, using~$|r_k(x_{1:T})| \le \varepsilon_\alpha/N$,
    \[
        \|R_\alpha\|_\infty
        \le \sum_{k \in \mathcal N} \mathbb E\left[|r_k(x_{1:T})| \|A_k(x_{1:T})\|_\infty\right]
        \le C \varepsilon_\alpha
        =
        O(\alpha)
    \]
    for small~$\alpha$ and a universal constant~$C$.
    Thus~$\nabla L=\alpha(\nabla L_{\mathrm{unif}}+O_\infty(\varepsilon_\alpha))$, so the perturbation from the non-uniform prediction softmax is negligible as~$\alpha \to 0$.

    Expanding the gradient~$\nabla M_k$, we have
    \begin{align*}
        \nabla M_k(z)^\top \Phi_2 x_t &= \frac{2}{N} \sum_j u_{k,j} (v_j^\top z) (v_j^\top \Phi_2 x_t) \\
        &= \frac{2}{N} \sum_{t'} \sigma_{t'} \sum_j u_{k,j} (v_j^\top \Phi_2 x_t) (v_j^\top \Phi_2 x_{t'}).
    \end{align*}

    We also denote~$\rho_2 = \sigma_{t^*} = \Theta(1/T)$, which is independent of the specific sequence under our assumptions of~$W_0$.

    We may write
    \begin{align*}
        A_k(x_{1:T})
        &= \frac{2}{N} \sum_{t,t'} \sigma_t \sigma_{t'} \sum_j u_{k,j} (v_j^\top \Phi_2 x_t) (v_j^\top \Phi_2 x_{t'}) \\
        &\quad \cdot (x_t - \bar x_{1:T}) x_T^\top \\
        &= 2 \rho_1^2 \delta_{k,m,m} (x_{T-1} - \bar x_{1:T}) x_T^\top \\
        &\quad + 2 \rho_2^2 \delta_{k,n,n} (x_{t^*} - \bar x_{1:T}) x_T^\top \\
        &\quad + 2 \rho_1 \rho_2 \delta_{k,m,n} (x_{t^*} + x_{T-1} - 2\bar x_{1:T}) x_T^\top,
    \end{align*}
    with~$\delta_{k,m,n} = \frac{1}{N} \sum_j u_{k,j} (v_j^\top \Phi_2 e_n) (v_j^\top \Phi_2 e_m) = \mathbf{1}\{k = m + n\}$, by noting that~$v_j^\top \Phi_2 x_t = 0$ for~$t \notin \{t^*, T-1\}$.

Based on the three terms above, the uniform part of the gradient satisfies
\begin{align*}
    \nabla_W L_{\mathrm{unif}} = G_1 + G_2 + G_3,
\end{align*}
with (using that~$y = n + m$)
\begin{align*}
    G_1 &= 2 \rho_1^2 \sum_{k \in \mathcal N} \mathbb E[(\frac{1}{N} - \mathbf{1}\{n + m = k\})\mathbf{1}\{k = 2m\}(x_{T-1} - \bar x_{1:T}) x_T^\top] \\
    &= 2 \rho_1^2 \mathbb E[(\frac{1}{N} - \mathbf{1}\{n = m\}) (x_{T-1} - \bar x_{1:T}) x_T^\top] \\
    &= 0,
\end{align*}
by noting that~$x_T$ is independent of~$n$ and~$m$, while each~$x_t$, $t < T$ only depends on either~$m$ or~$n$ (or neither), so that~$\mathbb E_{m,n}[\mathbf{1}\{m = n\} x_t | v] = \frac{1}{N} \mathbb E[x_t | v]$. Similarly,
\begin{align*}
    G_2 &= 2 \rho_2^2 \sum_{k \in \mathcal N} \mathbb E[(\frac{1}{N} - \mathbf{1}\{n + m = k\})\mathbf{1}\{k = 2n\}(x_{t^*} - \bar x_{1:T}) x_T^\top] \\
    &= 2 \rho_2^2 \mathbb E[(\frac{1}{N} - \mathbf{1}\{n = m\}) (x_{t^*} - \bar x_{1:T}) x_T^\top] \\
    &= 0.
\end{align*}
Finally, the last term contains the signal, as follows:
\begin{align*}
    G_3 &= 2 \rho_1 \rho_2 \sum_{k \in \mathcal N} \mathbb E[(\frac{1}{N} - \mathbf{1}\{n + m = k\})\mathbf{1}\{k = m + n\}(x_{t^*} + x_{T-1} - 2\bar x_{1:T}) x_T^\top] \\
    &= -2 \rho_1 \rho_2 \frac{N-1}{N} \mathbb E[(x_{t^*} + x_{T-1} - 2\bar x_{1:T}) x_T^\top] \\
    &= -2 \rho_1 \rho_2 \frac{N-1}{N} \mathbb E[(\Phi_1 e_v + e_n + p_{t^*} + (1 - 2 \rho_1) (e_m + p_{T-1}) \\
    &\quad - O(1/T) (e_n + e_m + e_v + \Phi_1 e_v + \sum_t p_t)) (\Phi_1 e_v + p_T)^\top] \\
    &= -2 \rho_1 \rho_2 \frac{N-1}{N} \left( \frac{1}{V} \sum_{v \in \mathcal V} \Phi_1 e_v (\Phi_1 e_v + p_T)^\top  - (2 \rho_1 - 1) p_{T-1} p_T^\top \right. \\
&\quad\quad \left. - \frac{2 \rho_1 - 1}{V} p_{T-1} \sum_v (\Phi_1 e_v)^\top + O_\infty\left(\frac{1}{N} + \frac{1}{T}\right) p_T^\top\right),
\end{align*}
where~$O_\infty(\epsilon)$ is a vector whose coordinates in the orthogonal basis of all embeddings are bounded by~$\epsilon$.
In particular, when~$N, T \gg V$, the last term is negligible, in the sense that it has little effect on attention logits after the gradient update. The first term is the associative-memory term, while the second and third terms subtract away some of the attention to the second operand that is already present for the model at~$W_0$.

Combining this display with~$\nabla L=\alpha(G_3+O_\infty(\varepsilon_\alpha))$, using~$\varepsilon_\alpha=O(\alpha)$ and~$(\rho_1\rho_2)^{-1}=O(T)$, and taking
\[
    \eta = \frac{V N}{2 \alpha \rho_1 \rho_2 (N-1)}
\]
gives
\begin{align*}
    W_1
    &= W_0 - \eta \nabla L(W_0) \\
    &= W_0
    + \sum_{v \in \mathcal V} \Phi_1 e_v(\Phi_1 e_v + p_T)^\top
    - (2\rho_1 - 1)V p_{T-1}p_T^\top \\
    &\quad
    - (2\rho_1 - 1)p_{T-1}\sum_{v \in \mathcal V}(\Phi_1 e_v)^\top
    + O_\infty\left(V\left(\frac{1}{N}+\frac{1}{T}+T\alpha\right)\right),
\end{align*}
where the final~$O_\infty(\cdot)$ is now interpreted in matrix infinity norm.
This proves the claimed expression for the normalized one-step update.

\end{proof}

\end{document}